\documentclass[conference]{IEEEtran}

\usepackage{graphicx} 
\usepackage{subfigure} 

\usepackage{algorithm}
\usepackage{algorithmic}

\usepackage[cmex10]{amsmath}
\usepackage{array}
\usepackage{amssymb, amsmath, amsfonts, latexsym, graphicx, tabularx, tabulary, upgreek, amsthm, dsfont}

\usepackage{xr}

\newtheorem{thm}{Theorem}
\newtheorem{lem}{Lemma}
\newtheorem{cor}{Corollary}
\theoremstyle{definition}

\newcommand{\E}{\mathbb{E}}

\newcommand{\argmin}{\operatornamewithlimits{argmin}}

\newcommand{\N}{\mathbb{N}}
\newcommand{\1}{\mathds{1}}
\newcommand{\logg}{\log}
\newcommand{\X}{X}

\begin{document}

\title{The Famine of Forte: \\Few Search Problems Greatly Favor Your Algorithm}

\author{\IEEEauthorblockN{George D. Monta\~nez}
\IEEEauthorblockA{Machine Learning Department\\
Carnegie Mellon University\\
Pittsburgh, Pennsylvania, USA\\
\texttt{gmontane@cs.cmu.edu}}
}

\maketitle

\begin{abstract}
Casting machine learning as a type of search, we demonstrate that the proportion of problems that are favorable for a fixed algorithm is strictly bounded, such that no single algorithm can perform well over a large fraction of them. Our results explain why we must either continue to develop new learning methods year after year or move towards highly parameterized models that are both flexible and sensitive to their hyperparameters. We further give an upper bound on the expected performance for a search algorithm as a function of the mutual information between the target and the information resource (e.g., training dataset), proving the importance of certain types of dependence for machine learning. Lastly, we show that the expected per-query probability of success for an algorithm is mathematically equivalent to a single-query probability of success under a distribution (called a \emph{search strategy}), and prove that the proportion of favorable strategies is also strictly bounded. Thus, whether one holds fixed the search algorithm and considers all possible problems or one fixes the search problem and looks at all possible search strategies, favorable matches are exceedingly rare. The forte (strength) of any algorithm is quantifiably restricted.
\end{abstract}

\section{Introduction}

\textbf{BIGCITY, United States.} In a fictional world not unlike our own, a sophisticated criminal organization plots to attack an unspecified landmark within the city. Due to the complexity of the attack and methods of infiltration, the group is forced to construct a plan relying on the coordinated actions of several interdependent agents, of which the failure of any one would cause the collapse of the entire plot. As a member of the city's security team, you must allocate finite resources to protect the many important locations within the city. Although you know the attack is imminent, your sources have not indicated which location, of the hundreds possible, will be hit; your lack of manpower forces you to make assumptions about target likelihoods. You know you can foil the plot if you stop even one enemy agent. Because of this, you seek to maximize the odds of capturing an agent by placing vigilant security forces at the strategic locations throughout the city. Allocating more security to a given location increases surveillance there, raising the probability a conspirator will be found if operating nearby. Unsure of your decisions, you allocate based on your best information, but continue to second-guess yourself.

With this picture in mind, we can analyze the scenario through the lens of algorithmic search. Algorithmic search methods, whether employed in fictional security situations or by researchers in the lab, all share common elements. They have a \textbf{search space} of elements (possible locations) which contains items of high value at unknown locations within that space. These high-value elements are \textbf{targets} of the search, which are sought by a process of inspecting elements of the search space to see if they contain any elements of the target set. (We refer to the inspection of search space elements as \textbf{sampling} from the search space, and each sampling action as a \textbf{query}.) In our fictional scenario, the enemy activity locations constitute the target set, and surveillance at city locations corresponds to a search process (i.e., attempting to locate and arrest enemy agents). The search process can be deterministic or contain some degree of randomization, such as choosing which elements to inspect using a weighted distribution over the search space. The search process is a \textbf{success} if an element of the target set is located during the course of inspecting elements of the space. The \textbf{history} of the search refers to the collection of elements already sampled and any accompanying information (e.g., the intelligence data gathered thus far by the security forces). 

There are many ways to search such a space. Research on algorithmic search is replete with proposed search methods that attempt to increase the expected likelihood of search success, across many different problem domains. However, a direct result of the No Free Lunch theorems are that no search method is universally optimal across all search problems~\cite{WOLPERTMACREADY,culberson1998futility,english1996evaluation,CUP,whitley2000functions,NO-MORE-LUNCH,droste2002optimization,dembski2010search,GENERAL-THEORY-OF-SEARCH}, since all algorithmic search methods have performance equivalent to uniform random sampling on the search space when uniformly averaged across any closed-under-permutation set of problem functions~\cite{CUP}. Thus, search (and learning) algorithms must trade weaker performance on some problems for improved performance on others~\cite{WOLPERTMACREADY,SCHAFFER1994}. Given the fact that there exists no universally superior search method, we typically seek to develop methods that perform well over a large range of important problems. An important question naturally arises: 
\begin{quote}
\emph{Is there a limit to the number of problems for which a search algorithm can perform well?}
\end{quote}
For search problems in general, the answer is \emph{yes}, as the proportion of search problems for which an algorithm outperforms uniform random sampling is strictly bounded in relation to the degree of performance improvement. Previous results have relied on the fact that most search problems have many target elements, and thus are not difficult for random sampling~\cite{MONTANEZ-TARGET,ENGLISH-LEARNING-HARD}. Since the typical search problem is not difficult for random sampling, it becomes hard to find a problem for which an algorithm can significantly and reliably outperform random sampling. While theoretically interesting, a more relevant situation occurs when we consider difficult search problems, which are those having relatively small target sets. We denote such problems as \textbf{target-sparse problems} or following~\cite{MONTANEZ-TARGET}, \textbf{target-sparse functions}. Note, this use of sparseness differs from that typical in the machine learning literature, since it refers to sparseness of the target space, not sparseness in the feature space. One focus of this paper is to prove results that hold even for target-sparse functions. As mentioned, Monta\~nez~\cite{MONTANEZ-TARGET} and English~\cite{ENGLISH-LEARNING-HARD} have shown that uniform sampling does well on the majority of functions since they are not target-sparse, thus making problems for which an algorithm greatly outperforms random chance rare. If we restrict ourselves to \emph{only} target-sparse functions, perhaps it would become easier to find relatively favorable problems since we would already know uniform sampling does poorly. The bar would already be low, so we would have to do less to surpass it, perhaps leading to a greater proportion of favorable problems within that set. We show why this intuition incorrect, along with proving several other key results.

\section{Contributions}

First, we demonstrate that favorable search problems must necessarily be rare. Our work departs from No Free Lunch results (namely, that the mean performance across sets of problems is fixed for all algorithms) to show that the proportion of favorable problems is strictly bounded in relation to the inherent problem difficulty and the degree of improvement sought (i.e., not just the mean performance is bounded). Our results continue to hold for sets of objective functions that are \emph{not} closed-under-permutation, in contrast to traditional No Free Lunch theorems. Furthermore, the negative results presented here do not depend on any distributional assumptions on the space of possible problems, such that the proportion of favorable problems is small regardless of which distribution holds over them in the real world. This directly answers critiques aimed at No Free Lunch results arguing against a uniform distribution on problems in the real world (cf.~\cite{rao1995every}), since given \emph{any} distribution over possible problems, there are still relatively few favorable problems within the set one is taking the distribution over.

As a corollary, we prove the information costs of finding any favorable search problem is bounded below by the number of bits ``advantage'' gained by the algorithm on such problems. We do this by using an active information transform to measure performance improvement in bits~\cite{COST-OF-SUCCESS}, proving a conservation of information result~\cite{NO-MORE-LUNCH,COST-OF-SUCCESS,MONTANEZ-TARGET} that shows the amount of information required to locate a search problem giving $b$ bits of expected information is at least $b$ bits. Thus, to get an expected net gain of information, the true distribution over search problems must be biased towards favorable problems for a given algorithm. This places a floor on the minimal information costs for finding favorable problems, somewhat reminiscent of the entertainingly satirical work on ``data set selection''~\cite{laloudouana2003data}.

Another major contribution of this paper is to bound the expected per-query probability of success based on information resources. Namely, we relate the degree of dependence (measured in mutual information) between target sets and external information resources, such as objective functions, noisy measurements\footnote{Noisy in the sense of inaccurate or biased, not in the sense of indeterministic.} or sets of training data, to the maximum improvement in search performance. We prove that for a fixed target-sparseness and given an algorithm with induced single-query probability of success $q$,
\begin{align}
    q \leq \frac{I(T; F) + D(P_T\|\mathcal{U}_T) + 1}{I_{\Omega}}
\end{align}
where $I(T; F)$ is the mutual information between target set $T$ (as a random variable) and external information resource $F$, $D(P_T\|\mathcal{U}_T)$ is the Kullback-Leibler divergence between the marginal distribution on $T$ and the uniform distribution on target sets, and $I_{\Omega}$ is the baseline information cost for the search problem due to sparseness. This simple equation takes into account degree of dependence, target sparseness, target function uncertainty, and the contribution of random luck. It is surprising that such well-known quantities appear in the course of simply trying to upper-bound the probability of success.

We then establish the equivalence between the expected per-query probability of success for an algorithm and the probability of a successful single-query search under some distribution, which we call a \emph{strategy}. Each algorithm maps to a strategy, and we prove an upper-bound on the proportion of favorable strategies for a fixed problem. Thus, matching a search problem to a fixed algorithm or a search algorithm to a fixed problem are both provably difficult tasks, and the set of favorable items remains vanishingly small in both cases.

Lastly, we apply the results to several problem domains, some toy and some actual, showing how these results lead to new insights in different research areas. 

\section{Search Framework}

We begin by formalizing our problem setting, search method abstraction and other necessary concepts.

\subsection{The Search Problem}\label{SEARCH-PROBLEM}
\begin{figure}
\centering
\includegraphics[width=\linewidth]{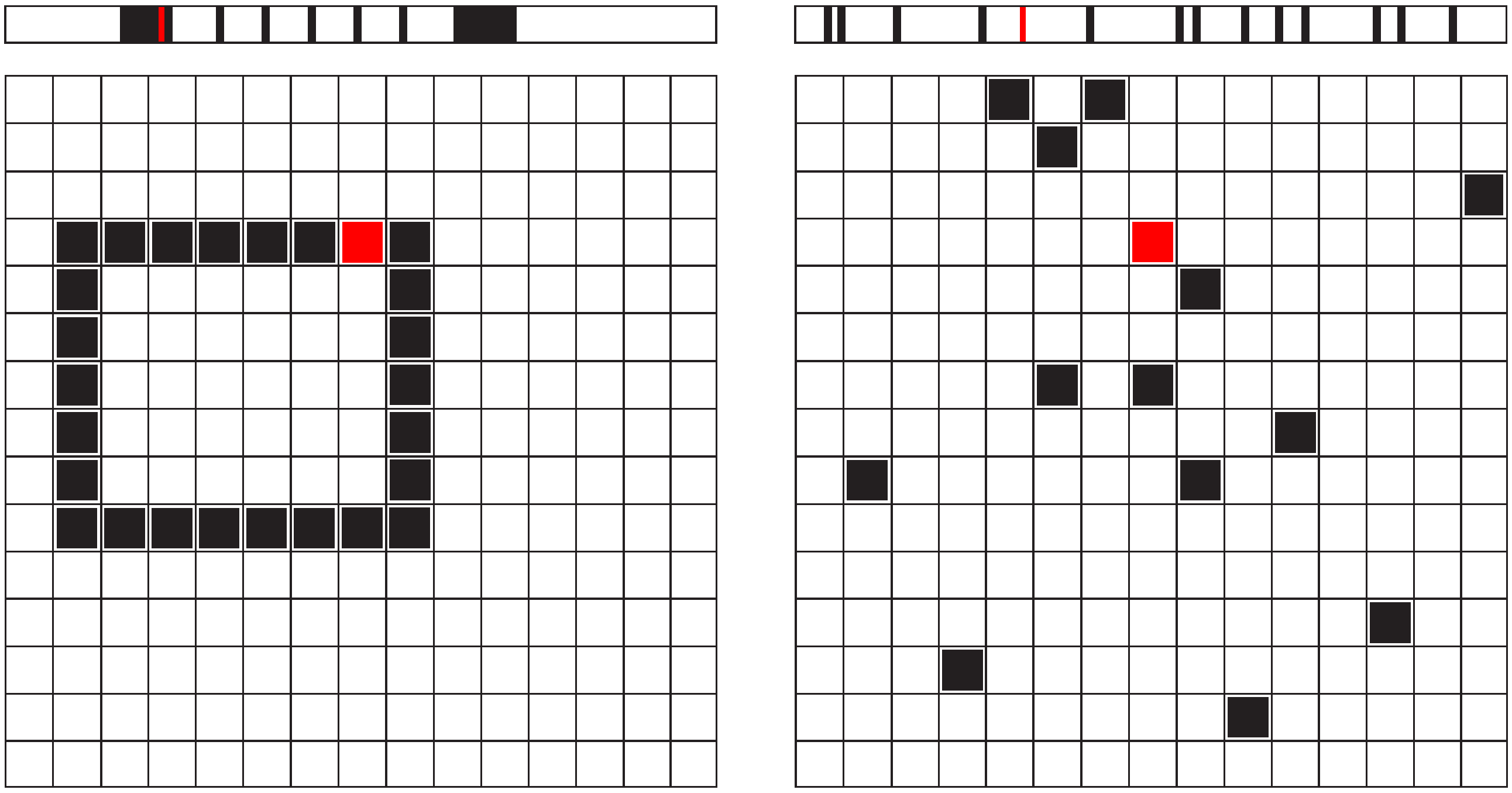}
\caption{Two example target sets. In one case (left), knowing the location of the black target elements fully determines the location of the red target element. In the second case (right), the locations of the black target elements reveal nothing concerning the location of the red element. Problems are represented as a binary vector (top row), and two-dimensional search space (bottom row).}
\label{fig:TWO-PROBLEMS}
\end{figure}

We abstract the search problem as follows. The search space, denoted $\Omega$, contains the elements to be examined. We limit ourselves to finite, discrete search spaces, which entails no loss of generality when considering search spaces fully representable on physical computer hardware within a finite time. Let the \textbf{target set} $T \subset \Omega$ be a nonempty subset of the search space. The set $T$ can be represented using a binary vector the size of $\Omega$, where an indexed position evaluates to 1 whenever the corresponding element is in $T$ and 0 otherwise. Thus, each $T$ corresponds to exactly one binary vector of length $|\Omega|$, and vice versa. We refer to this one-to-one mapped binary vector as a \textbf{target function} and use the terms target function and target set interchangeably, depending on context. These target sets/functions will help us define our space of possible search problems, as we will see shortly.

Figure~\ref{fig:TWO-PROBLEMS} shows two example target sets, in binary vector and generic target set form. The set on the left has strong (potentially exploitable) structure governing the placement of target elements, whereas the example on the right is more or less random. Thus, knowing the location of some target elements may or may not be able to help one find additional elements, and in general there may be any degree of correlation between the location of target elements already uncovered and those yet to be discovered. 

Typically, elements from the search space $\Omega$ are evaluated according to some \textbf{external information resource}, such as an objective function $f$. We abstract this resource as simply a finite length bit string, which could represent an objective function, a set of training data (in supervised learning settings), or anything else. The resource can exist in coded form, and we make no assumption about the shape of this resource or its encoding. Our only requirement is that it can be used as an oracle, given an element (possibly the null element) from $\Omega$. In other words, we require the existence of two methods, one for initialization (given a null element) and one for evaluating queried points in $\Omega$. We assume both methods are fully specified by the problem domain and external information resource. With a slight abuse of notation, we define an \textbf{information resource evaluation} as $F(\omega) := g(F,\omega)$, where $g$ is an extraction function applied to the information resource and $\omega \in \Omega \cup \{\emptyset\}$. Therefore, $F(\emptyset)$ represents the method used to extract initial information for the search algorithm (absent of any query), and $F(\omega)$ represents the evaluation of point $\omega$ under resource $F$. The size of the information resource becomes important for datasets in machine learning, which determines the maximum amount of mutual information available between $T$ and $F$. 

A \textbf{search problem} is defined as a 3-tuple, $(\Omega, T, F)$, consisting of a search space, a target subset of the space and an external information resource $F$, respectively. Since the target locations are hidden, any information gained by the search concerning the target locations is mediated through the external information resource $F$ alone. Thus, the space of possible search problems includes many deceptive search problems, where the external resource provides misleading information about target locations, similar to when security forces are given false intelligence data, and many noisy problems, similar to when imprecise intelligence gathering techniques are used. In the fully general case, there can be any relationship between $T$ and $F$. Because we consider any and all degrees of dependence between external information resources and target locations, this effectively creates independence when considering the set as a whole, allowing the first main result of this paper to follow as a consequence.

However, in many natural settings, target locations and external information resources are tightly coupled. For example, we typically threshold objective function values to designate the target elements as those that meet or exceed some minimum. Doing so enforces dependence between target locations and objective functions, where the former is fully determined by the latter, once the threshold is known. This dependence causes direct correlation between the objective function (which is observable) and the target locations (which are not directly observable). We will demonstrate such correlation is exploitable, affecting the upper bound on the expected probability of success.

\subsection{The Search Algorithm}
\begin{figure}
\centering
\includegraphics[width=\linewidth]{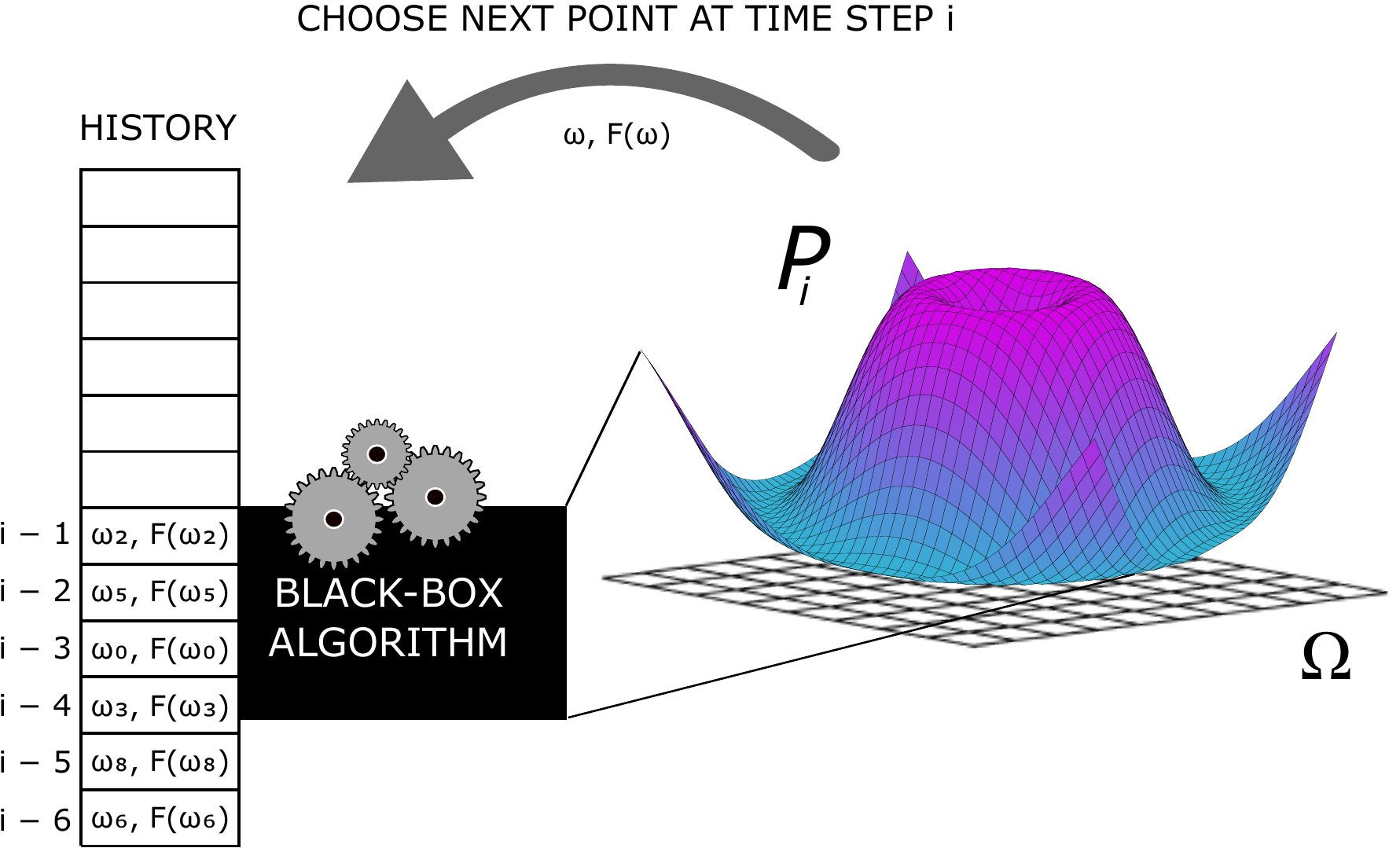}
\caption{Black-box search algorithm. At time $i$ the algorithm computes a probability distribution $P_i$ over the search space $\Omega$, using information from the history, and a new point is drawn according to $P_i$. The point is evaluated using external information resource $F$. The tuple $( \omega, F(\omega))$ is then added to the history at position $i$. Note, indices on $\omega$ elements do not correspond to time step in this diagram, but to sampled locations.}
\label{BLACKBOX-SEARCH}
\end{figure}

An algorithmic search is any process that chooses elements of a search space to examine, given some history of elements already examined and information resource evaluations. The history consists of two parts: a query trace and a resource evaluation trace. A \textbf{query trace} is a record of points queried, indexed by time. A \textbf{resource evaluation trace} is a record of partial information extracted from $F$, also indexed by time. The information resource evaluations can be used to build elaborate predictive models (as is the case of Bayesian optimization methods), or ignored completely (as is done with uniform random sampling). The algorithm's internal state is updated at each time step (according to the rules of the algorithm), as new query points are evaluated against the external information resource. The search is thus an iterative process, with the history $h$ at time $i$ represented as $h_{i} = (\omega_i, F(\omega_i))$. The problem domain defines how much initial information from $F$ is given by $F(\emptyset)$, as well as how much (and what) information is extracted from $F$ at each query.

We allow for both deterministic and randomized algorithms, since deterministic algorithms are equivalent to randomized algorithms with degenerate probability functions (i.e., they place all probability mass on a single point). Furthermore, any population based method can also be represented in this framework, by holding $P$ fixed when selecting the $m$ elements in the population, then considering the history (possibly limiting the horizon of the history to only $m$ steps, creating a Markovian dependence on the previous population) and updating the probability distribution over the search space for the next $m$ steps.

Abstracting away the details of how such a distribution $P$ is chosen, we can treat the search algorithm as a black box that somehow chooses elements of a search space to evaluate. A search is successful if an element in the target set is located during the search. Algorithm~\ref{alg:algorithmic_search} outlines the steps followed by the black-box search algorithm and Figure~\ref{BLACKBOX-SEARCH} visually demonstrates the process.
\begin{algorithm}
\caption{Black-box Search Algorithm}
\small
\begin{algorithmic}[1]
    \STATE Initialize $h_{0} \leftarrow (\emptyset, F(\emptyset))$.
\FORALL{$i = 1,\ldots,i_{\text{\tiny max}}$}
\STATE Using history $h_{0:i-1}$, compute $P_i$, the distribution over $\Omega$.
\STATE Sample element $\omega_i$ according to $P_i$.
\STATE Set $h_{i} \leftarrow (\omega_i, F(\omega_i))$.
\ENDFOR
\IF{an element of $T$ is contained in any tuple of $h$}
    \STATE Return \emph{success}.
\ELSE
    \STATE Return \emph{failure}.
\ENDIF
\end{algorithmic}
\label{alg:algorithmic_search}
\end{algorithm}

Our framework is sufficiently general to apply to supervised learning, genetic algorithms, and sequential model-based hyperparameter optimization, yet is specific enough to allow for precise quantitative statements. Although faced with an inherent balancing act whenever formulating a new framework, we emphasize greater generality to allow for the largest number of domain specific applications.

\section{Measuring Performance}\label{sec:MEASURING-PERFORMANCE}

Since general search algorithms may vary the total number of sampling steps performed, we measure performance using the expected per-query probability of success,
\begin{align}
 q(T, F) &= \mathbb{E}_{\tilde{P}, H} \left[ \frac{1}{|\tilde{P}|} \sum_{i=1}^{|\tilde{P}|} P_{i}(\omega \in T) \;\bigg|\; F\right], 
\end{align}
where $\tilde{P}$ is the sequence of probability distributions for sampling elements (with one distribution $P_i$ for each time step $i$), $H$ is the search history, and the $T$ and $F$ make the dependence on the search problem explicit. Because this is an expectation over sequences of history tuples and probability distributions, past information forms the basis for $q(T,F)$. $|\tilde{P}|$ denotes the length of the sequence $\tilde{P}$, which equals the number of queries taken. The expectation is taken over all sources of randomness, which includes randomness over possible search histories and any randomness in constructing the various $P_i$ from $h_{0:i-1}$ (if such a construction is not entirely deterministic). Taking the expectation over all sources of randomness is equivalent to the probability of success for samples drawn from an appropriately averaged distribution $\overline{P}(\cdot \mid F)$ (see Lemma~\ref{EXPECTED-PER-QUERY-PERFORMANCE}). Because we are sampling from a fixed (after conditioning and expectation) probability distribution, the expected per-query probability of success for the algorithm is equivalent to the induced amount of probability mass allocated to target elements. Revisiting our fictional security scenario, the probability of capturing an enemy agent is proportional to the amount of security at his location. In the same way, the expected per-query probability of success is equal to the amount of induced probability mass at locations where target elements are placed. Thus, each $\overline{P}(\cdot \mid F)$ demarks an equivalence class of search algorithms mapping to the same averaged distribution; we refer to these equivalence classes as \textbf{search strategies}.

We use uniform random sampling with replacement as our baseline search algorithm, which is a simple, always available strategy and we define $p(T, F)$ as the per-query probability of success for that method. In a natural way, $p(T, F)$ is a measure of the intrinsic difficulty of a search problem~\cite{COST-OF-SUCCESS}, absent of any side-information. The ratio $p(T, F)/q(T,F)$ quantifies the improvement achieved by a search algorithm over simple random sampling. Like an error quantity, when this ratio is small (i.e., less than one) the performance of the algorithm is better than uniform random sampling, but when it is larger than one, performance is worse. We will often write $p(T,F)$ simply as $p$, when the target set and information resource are clear from the context.

\section{Main Results}
We state our main results here with brief discussion, and with proofs given in the Appendix.

\subsection{Famine of Forte}\label{sec:FAVORABLE-FUNCTIONS}
\begin{thm}(Famine of Forte)\label{PROBABILITY-GOOD-SEARCH-PROBLEM}
    Define 
\[
    \uptau_{k} = \{T \mid T \subset \Omega, |T| = k \in \N \}
\] 
and let $\mathcal{B}_m$ denote any set of binary strings, such that the strings are of length $m$ or less. Let 
\begin{align*}
    R &= \{(T, F) \mid T \in \uptau_{k}, F \in \mathcal{B}_m \}, \text{ and } \\
    R_{q_{\text{min}}} &= \{(T, F) \mid T \in \uptau_{k}, F \in \mathcal{B}_m, q(T,F) \geq q_{\text{min}} \},
\end{align*}
where $q(T,F)$ is the expected per-query probability of success for algorithm $\mathcal{A}$ on problem $(\Omega, T, F)$. Then for any $m \in \N$,
\begin{align*}
    \frac{|R_{q_{\text{min}}}|}{|R|} &\leq \frac{p}{q_{\text{min}}}
\end{align*}
and 
\begin{align*}
    \lim_{m\rightarrow\infty}\frac{|R_{q_{\text{min}}}|}{|R|} &\leq \frac{p}{q_{\text{min}}}
\end{align*}
where $p=k/|\Omega|$.
\end{thm}
We see that for small $p$ (problems with sparse target sets) favorable search problems are rare if we desire a strong probability of success. The larger $q_\text{min}$, the smaller the proportion. In many real-world settings, we are given a difficult search problem (with minuscule $p$) and we hope that our algorithm has a reasonable chance of achieving success within a limited number of queries. According to this result, the proportion of problems fulfilling such criteria is also minuscule. Only if we greatly relax the minimum performance demanded, so that $q_\text{min}$ approaches the scale of $p$, can we hope to easily stumble upon such accommodating search problems. If not, insight and/or auxiliary information (e.g., convexity constraints, smoothness assumptions, domain knowledge, etc.) are required to find problems for which an algorithm excels.

This result has many applications, as will be shown in Section~\ref{sec:EXAMPLES}.

\subsection{Conservation of Information}\label{CONSERVATION-OF-INFORMATION}
\begin{cor}(Conservation of Active Information of Expectations)\label{cor:CONSERVATION-ACTIVE-INFO}
Define $I_{q(T,F)} = -\logg p/q(T,F)$ and let 
\begin{align*}
    R &= \{(T, F) \mid T \in \uptau_{k}, F \in \mathcal{B}_m \}, \text{ and } \\ 
    R_{b} &= \{(T, F) \mid T \in \uptau_{k}, F \in \mathcal{B}_m, I_{q(T,F)} \geq b \}. 
\end{align*}
Then for any $m \in \N$
\begin{align*}
    \frac{|R_{b}|}{|R|} &\leq 2^{-b}.
\end{align*}
\end{cor}
The active information transform quantifies improvement in a search over a baseline search method~\cite{COST-OF-SUCCESS}, measuring gains in search performance in terms of information (bits). This allows one to precisely quantify the proportion of favorable functions in relation to the number $b$ bits improvement desired. Active information provides a nice geometric interpretation of improved search performance, namely, that the improved search method is equivalent to a uniform random sampler on a reduced search space. Similarly, a degraded-performance search is equivalent to a uniform random sampler (with replacement) on a larger search space. 

Applying this transform, we see that finding a search problem for which an algorithm effectively reduces the search space by $b$ bits requires at least $b$ bits, so information is conserved in this context. Assuming you have no domain knowledge to guide the process of finding a search problem for which your search algorithm excels, you are unlikely to stumble upon one by chance; indeed, they are exponentially rare in the amount of improvement sought.

\subsection{Famine of Favorable Strategies}
\begin{thm}\label{SATISFYING-STRATEGY-IMPLICATIONS}(Famine of Favorable Strategies)
    For any fixed search problem $(\Omega, t, f)$, set of probability mass functions $\mathcal{P} = \{P : P \in \mathbb{R}^{|\Omega|}, \sum_j P_j = 1\}$, and a fixed threshold $q_\text{min} \in [0, 1]$,
    \[
        \frac{\mu(\mathcal{G}_{t, q_\text{min}})}{\mu(\mathcal{P})} \leq \frac{p}{q_\text{min}},
    \]
    where $\mathcal{G}_{t, q_\text{min}} = \{P : P \in \mathcal{P}, t^\top P \geq q_\text{min} \}$ and $\mu$ is Lebesgue measure. Furthermore, the proportion of possible search strategies giving at least $b$ bits of active information of expectations is no greater than $2^{-b}$. 
\end{thm}
Thus, not only are favorable problems rare, so are favorable strategies. Whether you hold fixed the algorithm and try to match a problem to it, or hold fixed the problem and try to match a strategy to it, both are provably difficult. Because matching problems to algorithms is hard, seemingly serendipitous agreement between the two calls for further explanation, serving as evidence against blind matching by independent mechanisms (especially for very sparse targets embedded in very large spaces). More importantly, this result places hard quantitative constraints (similar to minimax bounds in statistical learning theory) on information costs for automated machine learning, which attempts to match learning algorithms to learning problems~\cite{thornton2013auto, AUTOML}.

\subsection{Success Under Dependence}\label{MUTUAL-INFO-RESULT}
\begin{thm}(Success Under Dependence)\label{MUTUAL-INFORMATION-BOUND}
Define 
\[
    q = \E_{T, F}\left[q(T,F)\right]
\]
and note that
\[
    q = \E_{T, F}\left[\;\overline{P}(\omega \in T | F) \right] = \Pr(\omega \in T; \mathcal{A}).
\]
Then, 
\begin{align*}
        q \leq \frac{I(T; F) + D(P_T \| \mathcal{U}_T) + 1}{I_{\Omega}}
    \end{align*}
    where $I_{\Omega} = -\logg k/|\Omega|$, $D(P_T \| \mathcal{U}_T)$ is the Kullback-Leibler divergence between the marginal distribution on $T$ and the uniform distribution on $T$, and $I(T; F)$ is the mutual information. Alternatively, we can write
\begin{align*}
    q \leq \frac{H(\mathcal{U}_T) - H(T \mid F) + 1}{I_{\Omega}}
\end{align*}
where $H(\mathcal{U}_T) = \logg\binom{|\Omega|}{k}$.
\end{thm}
Thus the bound on expected probability of success improves monotonically with the amount of dependence between target sets and information resources. We quantify the dependence using mutual information under any fixed joint distribution on $\uptau_{k}$ and $\mathcal{B}_m$. We see that $I_{\Omega}$ measures the relative target sparseness of the search problem, and can be interpreted as the information cost of locating a target element in the absence of side-information. $D(P_T\|\mathcal{U}_T)$ is naturally interpreted as the predictability of the target sets, since large values imply the probable occurrence of only a small number of possible target sets. The mutual information $I(T;F)$ is the amount of exploitable information the external resource contains regarding $T$; lowering the mutual information lowers the maximum expected probability of success for the algorithm. Lastly, the $1$ in the numerator upper bounds the contribution of pure randomness, as explained in the Appendix. Thus, this expression constrains the relative contributions of predictability, problem difficulty, side-information, and randomness for a successful search, providing a satisfyingly simple and interpretable upper bound on the probability of successful search.

\section{Examples}\label{sec:EXAMPLES}

\subsection{Binary Classification}

It has been suggested that machine learning represents a type of search through parameter or concept space~\cite{mitchell1982generalization}. Supporting this view, we can represent binary classification problems within our framework as follows:
\begin{itemize}
    \item $\mathcal{A}$ - classification algorithm, such as an SVM.        
    \item $\Omega$ - space of possible concepts over an instance space.
    \item $T$ - Set of all hypotheses with less than 10\% classification error on test set, for example.  
    \item $F$ - set of training examples.
        \begin{itemize}
           \item $F(\emptyset)$ - full set of training data.
           \item $F(c)$ - loss on training data for concept $c$.
        \end{itemize}
    \item $(\Omega, T, F)$ - binary classification learning task.
\end{itemize}
The space of possible binary concepts is $\Omega$, with the true concept being an element in that space. In our example, let $|\Omega| = 2^{100}$. The target set consists of the set of all concepts in that space that (1) are consistent with the training data (which we will assume all are), and (2) differ from the truth in at most 10\% of positions on the generalization held-out dataset. Thus, $|T| = \sum_{i=0}^{10}\binom{100}{i}$. Let us assume the marginal distribution on $T$ is uniform, which isn't necessary but simplifies the calculation. The external information resource $F$ is the set of training examples. The algorithm uses the training examples (given by $F(\emptyset)$) to produce a distribution over the space of concepts; for deterministic algorithms, this is a degenerate distribution on exactly one element. A single query is then taken (i.e., a concept is output), and we assess the probability of success for the single query. By Theorem~\ref{MUTUAL-INFORMATION-BOUND}, the expected chances of outputting a concept with at least 90\% generalization accuracy is thus no greater than $\frac{I(T; F) + 1}{I_{\Omega}} \approx \frac{I(T; F)}{I_{\Omega}} \leq \frac{I(T; F)}{59}$. The denominator is the information cost of specifying at least one element of the target set and the numerator represents the information resources available for doing so. When the mutual information meets (or exceeds) that cost, success can be ensured for any algorithm perfectly mining the available mutual information. When noise reduces the mutual information below the information cost, the expected probability of success becomes strictly bounded in proportion to that ratio.

\subsection{General Learning Problems}

Vapnik presents a generalization of learning that applies to classification, regression, and density estimation~\cite{vapnik1999overview}, which we can translate into our framework. Following Vapnik, let $P(z)$ be defined on space $Z$, and consider the parameterized set of functions $Q_{\alpha}(z), \alpha \in \Lambda$. The goal is to minimize $R(\alpha) = \int Q_{\alpha}(z)dP(z)$ for $\alpha \in \Lambda$, when $P(z)$ is unknown but an i.i.d.\ sample $z_1,\ldots,z_{\ell}$ is given. Let $R_{emp}(\alpha) = \frac{1}{\ell}\sum_{i=1}^{\ell} Q_{\alpha}(z_i)$ be the empirical risk.

To reduce this general problem to a search problem within our framework, assume $\Lambda$ is finite, choose $\epsilon \in \mathbb{R}_{\geq 0}$, and let
\begin{itemize}
    \item $\Omega = \Lambda$;
    \item $T = \{\alpha : R(\alpha) - \argmin_{\alpha' \in \Lambda} R(\alpha') < \epsilon$\};
    \item $F = \{z_1, \ldots, z_{\ell}\}$;
    \item $F(\emptyset) = \{z_1, \ldots, z_{\ell}\}$; \text{ and }
    \item $F(\alpha) = R_{emp}(\alpha)$.
\end{itemize}
Thus, any finite problem representable in Vapnik's statistical learning framework is also directly representable within our search framework.

\subsection{Hyperparameter Optimization}

Given that sequential hyperparameter optimization is a literal search through a space of hyperparameter configurations, our results are directly applicable. The search space $\Omega$ consists of all the possible hyperparameter configurations (appropriately discretized in the case of numerical hyperparameters). The target set $T$ is determined by the particular learning algorithm the configurations are applied to, the performance metric used, and the level of performance desired. Let $X$ denote a set of points sampled from the space, and let the information gained from the sample become the external information resource $f$. Given that resource, we have the following theorem:
\begin{thm}
    Given a search algorithm $\mathcal{A}$, a finite discrete hyperparameter configuration space $\Omega$, a set $X$ of points sampled from that search space, and information resource $f$ that is a function of $X$, let $\Omega' := \Omega \setminus X$, $\uptau_k = \{T \mid T \subset \Omega', |T| = k \in \N \}$, and $\uptau_{k,q_{\text{min}}} = \{T \mid T \in \uptau_k, q(T,f) \geq q_{\text{min}} \}$, where $q(T, f)$ is the expected per-query probability of success for algorithm $\mathcal{A}$ under $T$ and $f$. Then,
\begin{align*}
    \frac{|\uptau_{k,q_{\text{min}}}|}{|\uptau_k|} &\leq \frac{p'}{q_{\text{min}}}
\end{align*}
where $p' = k/|\Omega'|$
\end{thm}
The proof follows directly from Theorem~\ref{PROBABILITY-GOOD-SEARCH-PROBLEM}. 

The proportion of possible hyperparameter target sets giving an expected probability of success $q_{\text{min}}$ or more is minuscule when $k \ll |\Omega'|$. If we have no additional information beyond that gained from the points $X$, we have no justifiable basis for expecting a successful search. Thus, we must make some assumptions concerning the relationship of the points sampled to the remaining points in $\Omega'$. We can do so by either assuming structure on the search space, such that spatial information becomes informative, or by making an assumption on the process by which $X$ was sampled, so that the sample is representative of the space in quantifiable ways. These assumptions allow $f$ to become informative of the target set $T$, leading to exploitable dependence under Theorem~\ref{MUTUAL-INFORMATION-BOUND}. Thus we see the need for inductive bias in hyperparameter optimization (to expand a term used by Mitchell~\cite{Mitchell1980} for classification), which hints at a strategy for creating more effective hyperparameter optimization algorithms (i.e., through exploitation of spatial structure).

\subsection{One-Size-Fits-All Fitness Functions}

Theorem~\ref{PROBABILITY-GOOD-SEARCH-PROBLEM} gives us an upper bound on the proportion of favorable search problems, but what happens when we have a single, fixed information resource, such as a single fitness function? A natural question to ask is for how many target locations can such a fitness function be useful. More precisely, for a given search algorithm, for what proportion of search space locations can the fitness function raise the expected probability of success, assuming a target element happens to be located at one of those spots? 

Applying Theorem~\ref{PROBABILITY-GOOD-SEARCH-PROBLEM} with $|T| = 1$, we find that a fitness function can significantly raise the probability of locating target elements placed on, at most, 1/$q_\text{min}$ search space elements. We see this as follows. Since $|T| = 1$ and the fitness function is fixed, each search problem maps to exactly one element of $\Omega$, giving $|\Omega|$ possible search problems. The number of $q_\text{min}$-favorable search problems is upper-bounded by
\[
    |\Omega|\frac{p}{q_\text{min}} = |\Omega|\frac{|T|/|\Omega|}{q_\text{min}}  = |\Omega|\frac{1/|\Omega|}{q_\text{min}} = \frac{1}{q_\text{min}}.
\]
Because this expression is independent of the size of the search space, the number of elements for which a fitness function can strongly raise the probability of success remains fixed even as the size of the search space increases. Thus, for very large search spaces the proportion of favored locations effectively vanishes. There can exist no single fitness function that is strongly favorable for many elements simultaneously, and thus no ``one-size-fits-all'' fitness function.

\subsection{Proliferation of Learning Algorithms}

Our results also help make sense of recent trends in machine learning. They explain why new algorithms continue to be developed, year after year, conference after conference, despite decades of research in machine learning and optimization. Given that algorithms can only perform well on a narrow subset of problems, we must either continue to devise novel algorithms for new problem domains or else move towards flexible algorithms that can modify their behavior solely though parameterization. The latter effectively behave like new strategies for new hyperparameter configurations. The explosive rise of flexible, hyperparameter sensitive algorithms like deep learning methods and vision architectures shows a definite trend towards the latter, with their hyperparameter sensitivity being well known~\cite{snoek2012practical,bergstra2013making}. Furthermore, because flexible algorithms are highly sensitive to their hyperparameters (by definition), this explains the concurrent rise in automated hyperparameter optimization methods~\cite{shahriari2016taking,bergstra2012random,thornton2013auto,wang-jair2016}.

\subsection{Landmark Security}

Returning to our initial toy example, we can now fill in a few details. Our external information resource is the pertinent intelligence data, mined through surveillance. We begin with background knowledge (represented in $F(\emptyset)$), used to make the primary security force placements. Team members on the ground communicate updates back to central headquarters, such as suspicious activity, which are the $F(\omega)$ evaluations used to update the internal information state. Each resource allocated is a query, and manpower constraints limit the number of queries available. Doing more with fewer officers is better, so the hope is to maximize the per-officer probability of stopping the attack.

Our results tell us a few things. First, a fixed strategy can only work well in a limited number of situations. There is little or no hope of a problem being a good match for your strategy if the problem arises independently of it (Theorems~\ref{PROBABILITY-GOOD-SEARCH-PROBLEM} and \ref{SATISFYING-STRATEGY-IMPLICATIONS}). So reliable intelligence becomes key. The better correlated the intelligence reports are with the actual plot, the better a strategy can perform (Theorem~\ref{MUTUAL-INFORMATION-BOUND}). However, even for a fixed search problem with reliable external information resource there is no guarantee of success, if the strategy is chosen poorly; the proportion of good strategies for a fixed problem is no better than the proportion of good problems for a fixed algorithm (Theorem~\ref{SATISFYING-STRATEGY-IMPLICATIONS}). Thus, domain knowledge is crucial in choosing either. Without side-information to guide the match between search strategy and search problem, the expected probability of success is dismal in target-sparse situations.

\section{Conclusion}

A colleague once remarked that the results presented here bring to mind Lake Wobegone, ``where all the children are above average.'' In a world where every published algorithm is ``above average,'' conservation of information reminds us that this cannot be. Improved performance over any subset of problems necessarily implies degraded performance over the remaining problems~\cite{SCHAFFER1994}, and all methods have performance equivalent to random sampling when uniformly averaged over any closed-under-permutation set of problems~\cite{WOLPERTMACREADY,culberson1998futility,CUP}.

But there are many ways to achieve an average. An algorithm might perform well over a large number of problems, only to perform poorly on a small set of remaining problems. A different algorithm might perform close to average on all problems, with slight variations around the mean. How large can the subset of problems with improved performance be? We show that the maximum \emph{proportion} of search problems for which an algorithm can attain $q_{\text{min}}$-favorable performance is bounded from above by $p/q_{\text{min}}$. Thus, an algorithm can only perform well over a narrow subset of possible problems. Not only is there no free lunch, but there is also a famine of favorable problems.

If finding a good search problem for a fixed algorithm is hard, then so is finding a good search algorithm for a fixed problem (Theorem~\ref{SATISFYING-STRATEGY-IMPLICATIONS}). Thus, the matching of problems to algorithms is provably difficult, regardless of which is fixed and which varies. 

Our results paint a more optimistic picture once we restrict ourselves to those search problems for which the external information resource is strongly informative of the target set, as is often assumed to be the case. For those problems, the expected per-query probability of success is upper bounded by a function involving the mutual information between target sets and external information resources (like training datasets and objective functions). The lower the mutual information, the lower chance of success, but the bound improves as the dependence is strengthened.

The search framework we propose is general enough to find application in many problem areas, such as machine learning, evolutionary search, and hyperparameter optimization. The results are not just of theoretical importance, but help explain real-world phenomena, such as the need for exploitable dependence in machine learning and the empirical difficulty of automated learning~\cite{AUTOML}. Our results help us understand the growing popularity of deep learning methods and unavoidable interest in automated hyperparameter tuning methods. Extending the framework to continuous settings and other problem areas (such as active learning) is the focus of ongoing research.

\section*{Acknowledgement}
I would like to thank Akshay Krishnamurthy, Junier Oliva, Ben Cowley and Willie Neiswanger for their discussions regarding Lemma~\ref{SUBSET-SELECTION-LEMMA}. I am indebted to Geoff Gordon for help proving Lemma~\ref{SUBSET-SELECTION-LEMMA}, and to Cosma Shalizi for providing many good insights, challenges and ideas concerning this manuscript.

\bibliographystyle{IEEEtran}
\bibliography{references}

\section{Appendix: Proofs}
\label{PROOFS}

\setcounter{thm}{0}
\setcounter{cor}{0}
\setcounter{lem}{0}

\begin{lem}\label{EXPECTED-PER-QUERY-PERFORMANCE}(Expected Per Query Performance From Expected Distribution) Let $t$ be a target set, $q(t, f)$ the expected per-query probability of success for an algorithm and $\nu$ be the conditional joint measure induced by that algorithm over finite sequences of probability distributions and search histories, conditioned on external information resource $f$. Denote a probability distribution sequence by $\tilde{P}$ and a search history by $h$. Let $\mathcal{U}(\tilde{P})$ denote a uniform distribution on elements of $\tilde{P}$ and define $\overline{P}(x\mid f) = \int \mathbb{E}_{P \sim \mathcal{U}(\tilde{P})}[P(x)] d\nu(\tilde{P}, h \mid f)$. Then,
\begin{align*}
    q(t, f) &= \overline{P}(X\in t | f)
\end{align*}
where $\overline{P}(X | f)$ is a probability distribution on the search space.
\end{lem}

\begin{proof}
    Begin by expanding the definition of $\mathbb{E}_{P \sim \mathcal{U}(\tilde{P})}[P(x)]$, being the average probability mass on element $x$ under sequence $\tilde{P}$:
    \begin{align*}
        \mathbb{E}_{P \sim \mathcal{U}(\tilde{P})}[P(x)] &= \frac{1}{|\tilde{P}|} \sum_{i=1}^{|\tilde{P}|} P_{i}(x).
    \end{align*}   
    We note that $\overline{P}(x|f)$ is a proper probability distribution, since
    \begin{enumerate}
        \item $\overline{P}(x|f) \geq 0$, being the integral of a nonnegative function;
        \item $\overline{P}(x|f) \leq 1$, as
             \begin{align*}
                \overline{P}(x|f) &\leq \int \left[\frac{1}{|\tilde{P}|} \sum_{i=1}^{|\tilde{P}|} 1 \right]d\nu(\tilde{P}, h | f) = 1;
             \end{align*}
        \item $\overline{P}(x|f)$ sums to one, because
            \begin{align*}
                 \sum_{x}\overline{P}(x\mid f) &= \sum_{x}\left[\int \frac{1}{|\tilde{P}|} \sum_{i=1}^{|\tilde{P}|}  P_{i}(x) d\nu(\tilde{P}, h|f) \right] \\
                                         &= \int \frac{1}{|\tilde{P}|} \sum_{i=1}^{|\tilde{P}|} \left[\sum_{x}P_{i}(x)\right] d\nu(\tilde{P}, h|f)  \\
                                         &= \int \frac{|\tilde{P}|}{|\tilde{P}|} d\nu(\tilde{P}, h|f) \\
                                         &= 1.
            \end{align*}
    \end{enumerate}
    Finally, 
    \begin{align*}
        \overline{P}(X\in t|f)
            &= \sum_{x}\1_{x\in t} \overline{P}(x|f) \\
            &= \sum_{x}\left[\1_{x\in t} \int \mathbb{E}_{P \sim \mathcal{U}(\tilde{P})}[P(x)] d\nu(\tilde{P}, h|f)\right] \\
            &= \sum_{x}\left[\1_{x\in t} \int \left[\frac{1}{|\tilde{P}|} \sum_{i=1}^{|\tilde{P}|} P_{i}(x) \right] d\nu(\tilde{P}, h|f)\right] \\
            &= \int \frac{1}{|\tilde{P}|} \sum_{i=1}^{|\tilde{P}|} \left[\sum_{x} \1_{x\in t} P_{i}(x) \right] d\nu(\tilde{P}, h|f) \\
            &= \mathbb{E}_{\tilde{P}, H} \left[ \frac{1}{|\tilde{P}|} \sum_{i=1}^{|\tilde{P}|} P_{i}(X\in t) \bigg| f \right]\\
            &= q(t, f).
    \end{align*}
\end{proof}

\begin{lem}\label{SUBSET-SELECTION-LEMMA}(Maximum Number of Satisfying Vectors)
    Given an integer $1 \leq k \leq n$, a set $\mathcal{S} = \{\mathbf{s} : \mathbf{s} \in \{0, 1\}^n, \|\mathbf{s}\| = \sqrt{k}\}$ of all $n$-length $k$-hot binary vectors, a set $\mathcal{P} = \{P : P \in \mathbb{R}^n, \sum_j P_j = 1\}$ of discrete $n$-dimensional simplex vectors, and a fixed scalar threshold $\epsilon \in [0, 1]$, then for any fixed $P \in \mathcal{P}$,
\[
    \sum_{\mathbf{s}\in \mathcal{S}} \1_{\mathbf{s}^\top P \geq \epsilon} \leq \frac{1}{\epsilon}\binom{n-1}{k-1}
\]
where $\mathbf{s}^\top P$ denotes the vector dot product between $\mathbf{s}$ and $P$. 
\end{lem}

\begin{proof}
    For $\epsilon = 0$, the bound holds trivially. For $\epsilon > 0$, let $S$ be a random quantity that takes values $\mathbf{s}$ uniformly in the set $\mathcal{S}$. Then, for any fixed $P \in \mathcal{P}$,
\begin{align*}
    \sum_{\mathbf{s}\in \mathcal{S}} \1_{\mathbf{s}^\top P \geq \epsilon} &= \binom{n}{k} \mathbb{E}\left[\1_{S^\top P \geq \epsilon}\right] \\
          &= \binom{n}{k} \Pr\left(S^\top P \geq \epsilon \right).
\end{align*}
    Let $\mathbf{1}$ denotes the all ones vector. Under a uniform distribution on random quantity $S$ and because $P$ does not change with respect to $\mathbf{s}$, we have
\begin{align*}
    \mathbb{E}\left[S^\top P\right] &= \binom{n}{k}^{-1}\sum_{\mathbf{s}\in \mathcal{S}}\mathbf{s}^\top P \\
         &= P^\top\binom{n}{k}^{-1}\sum_{\mathbf{s}\in \mathcal{S}}\mathbf{s} \\
         &= P^\top\frac{\mathbf{1}\binom{n-1}{k-1}}{\binom{n}{k}} \\
         &= P^\top\frac{\mathbf{1}\binom{n-1}{k-1}}{\frac{n}{k}\binom{n-1}{k-1}} \\
         &= \frac{k}{n}P^\top\mathbf{1} \\
         &= \frac{k}{n}
\end{align*}
since $P$ must sum to $1$.

Noting that $S^\top P \geq 0$, we use Markov's inequality to get
\begin{align*}
\sum_{\mathbf{s}\in \mathcal{S}} \1_{\mathbf{s}^\top P \geq \epsilon}
          &= \binom{n}{k} \Pr\left(S^\top P \geq \epsilon \right) \\
          &\leq \binom{n}{k} \frac{1}{\epsilon}\mathbb{E}\left[S^\top P\right] \\
          &= \binom{n}{k} \frac{1}{\epsilon} \frac{k}{n} \\
          &= \frac{1}{\epsilon}\binom{n-1}{k-1}.
\end{align*}
\end{proof}

\begin{thm}(Famine of Forte)
 Define 
\[
    \uptau_{k} = \{T \mid T \subset \Omega, |T| = k \in \N \}
\] 
and let $\mathcal{B}_m$ denote any set of binary strings, such that the strings are of length $m$ or less. Let 
\begin{align*}
    R &= \{(T, F) \mid T \in \uptau_{k}, F \in \mathcal{B}_m \}, \text{ and } \\
    R_{q_{\text{min}}} &= \{(T, F) \mid T \in \uptau_{k}, F \in \mathcal{B}_m, q(T,F) \geq q_{\text{min}} \},
\end{align*}
where $q(T,F)$ is the expected per-query probability of success for algorithm $\mathcal{A}$ on problem $(\Omega, T, F)$. Then for any $m \in \N$,
\begin{align*}
    \frac{|R_{q_{\text{min}}}|}{|R|} &\leq \frac{p}{q_{\text{min}}}
\end{align*}
and 
\begin{align*}
    \lim_{m\rightarrow\infty}\frac{|R_{q_{\text{min}}}|}{|R|} &\leq \frac{p}{q_{\text{min}}}
\end{align*}
where $p=k/|\Omega|$.
\end{thm}

\begin{proof}
	We begin by defining a set $\mathcal{S}$ of all $|\Omega|$-length target functions with exactly $k$ ones, namely, $\mathcal{S} = \{\mathbf{s} : \mathbf{s} \in \{0, 1\}^{|\Omega|}, \|\mathbf{s}\| = \sqrt{k}\}$. For each of these, we have $|\mathcal{B}_m|$ external information resources. The total number of search problems is therefore
\begin{align}\label{NUMBER-OF-SEARCH-PROBLEMS}
    \binom{|\Omega|}{k}|\mathcal{B}_m|.
\end{align}
We seek to bound the proportion of possible search problems for which $q(\mathbf{s}, f) \geq q_\text{min}$ for any threshold $q_\text{min} \in (0, 1]$. Thus,
\begin{align}
    \frac{|R_{q_{\text{min}}}|}{|R|} &\leq \frac{|\mathcal{B}_m| \sup_{f}\left[\sum_{\mathbf{s} \in \mathcal{S}} \1_{q(\mathbf{s}, f) \geq q_\text{min}}\right] }{|\mathcal{B}_m|\binom{|\Omega|}{k}} \\
                                     &= \binom{|\Omega|}{k}^{-1}\sum_{\mathbf{s} \in \mathcal{S}} \1_{q(\mathbf{s}, f^*) \geq q_\text{min}},
\end{align}
where $f^* \in \mathcal{B}_m$ denotes the arg sup of the expression. Therefore,
\begin{align*}
    \frac{|R_{q_{\text{min}}}|}{|R|} &\leq \binom{|\Omega|}{k}^{-1}\sum_{\mathbf{s} \in \mathcal{S}} \1_{q(\mathbf{s}, f^*) \geq q_\text{min}} \\
        &= \binom{|\Omega|}{k}^{-1}\sum_{\mathbf{s} \in \mathcal{S}} \1_{\overline{P}(\omega \in \mathbf{s}|f^*) \geq q_\text{min}} \\
        &= \binom{|\Omega|}{k}^{-1}\sum_{\mathbf{s} \in \mathcal{S}} \1_{\mathbf{s}^\top\overline{P}_{f^*} \geq q_\text{min}}
\end{align*}
where the first equality follows from Lemma~\ref{EXPECTED-PER-QUERY-PERFORMANCE}, $\omega \in \mathbf{s}$ means the target function $\mathbf{s}$ evaluated at $\omega$ is one, and $\overline{P}_{f^*}$ represents the $|\Omega|$-length probability vector defined by $\overline{P}(\cdot|f^*)$. By Lemma~\ref{SUBSET-SELECTION-LEMMA}, we have
\begin{align}
\binom{|\Omega|}{k}^{-1}\sum_{\mathbf{s} \in \mathcal{S}} \1_{\mathbf{s}^\top\overline{P}_{f^*} \geq q_\text{min}} 
&\leq \binom{|\Omega|}{k}^{-1} \left[\frac{1}{q_\text{min}}\binom{|\Omega|-1}{k-1}\right] \notag{}\\
&= \frac{k}{|\Omega|}\frac{1}{q_\text{min}} \notag{}\\
&= p / q_\text{min}\label{eq:final-bound-thm-pgsp}
\end{align}
proving the result for finite external information resources.

To extend to infinite external information resources, let $A_m = \{f : f \in \{0,1\}^{\ell}, \ell \in \N, \ell \leq m\}$ and define
\begin{align}
    a_m &:= \frac{|A_m| \sup_{f \in A_m}\left[\sum_{\mathbf{s} \in \mathcal{S}} \1_{q(\mathbf{s}, f) \geq q_\text{min}} \right]}{|A_m|\binom{|\Omega|}{k}}, \\
    b_m &:= \frac{|\mathcal{B}_m| \sup_{f \in \mathcal{B}_m}\left[\sum_{\mathbf{s} \in \mathcal{S}} \1_{q(\mathbf{s}, f) \geq q_\text{min}} \right]}{|\mathcal{B}_m|\binom{|\Omega|}{k}}.
\end{align}
We have shown that $a_m \leq p / q_\text{min}$ for each $m \in \N$. Thus,
\begin{align*}
\limsup_{m\rightarrow\infty} \frac{|A_m| \sup_{f \in A_m}\left[\sum_{\mathbf{s} \in \mathcal{S}} \1_{q(\mathbf{s}, f) \geq q_\text{min}} \right]}{|A_m|\binom{|\Omega|}{k}}
    &= \limsup_{m\rightarrow\infty} a_m \\
    &\leq \sup_m a_m \\
    &\leq p / q_\text{min}.
\end{align*}
Next, we use the monotone convergence theorem to show the limit exists. First,
\begin{align}
    \lim_{m\rightarrow\infty} a_m 
        &= \lim_{m\rightarrow\infty} \frac{\sup_{f \in A_m}\left[\sum_{\mathbf{s}} \1_{q(\mathbf{s}, f) \geq q_\text{min}} \right]}{\binom{|\Omega|}{k}}
\end{align}
By construction, the successive $A_m$ are nested with increasing $m$, so the sequence of suprema (and numerator) are increasing, though not necessarily strictly increasing. The denominator is not dependent on $m$, so $\{a_m\}$ is an increasing sequence. Because it is also bounded above by $p / q_\text{min}$, the limit exists by monotone convergence. Thus,
\[
    \lim_{m\rightarrow\infty} a_m = \limsup_{m\rightarrow\infty} a_m \leq p / q_\text{min}.
\]
Lastly, 
\begin{align*}
    \lim_{m\rightarrow\infty} b_m
    &= \lim_{m\rightarrow\infty} \frac{|\mathcal{B}_m| \sup_{f \in \mathcal{B}_m}\left[\sum_{\mathbf{s} \in \mathcal{S}} \1_{q(\mathbf{s}, f) \geq q_\text{min}} \right]}{|\mathcal{B}_m|\binom{|\Omega|}{k}} \\
    &= \lim_{m\rightarrow\infty} \frac{\sup_{f \in \mathcal{B}_m}\left[\sum_{\mathbf{s} \in \mathcal{S}} \1_{q(\mathbf{s}, f) \geq q_\text{min}} \right]}{\binom{|\Omega|}{k}} \\
    &\leq \lim_{m\rightarrow\infty} \frac{\sup_{f \in A_m}\left[\sum_{\mathbf{s} \in \mathcal{S}} \1_{q(\mathbf{s}, f) \geq q_\text{min}} \right]}{\binom{|\Omega|}{k}} \\
    &= \lim_{m\rightarrow\infty} a_m\\
    &\leq p / q_\text{min}.
\end{align*}
\end{proof}

\begin{cor}(Conservation of Active Information) \label{cor:CONSERVATION-ACTIVE-INFORMATION-EXPECTATIONS}
Define \emph{active information of expectations} as $I_{q(T,F)} = -\logg p/q(T,F)$, where $p$ is the per-query probability of success for uniform random sampling and $q(T,F)$ is the expected per-query probability of success for an alternative search algorithm. Define 
\[
    \uptau_{k} = \{T \mid T \subset \Omega, |T| = k \in \N\}
\] 
and let $\mathcal{B}_m$ denote any set of binary strings, such that the strings are of length $m$ or less. Let 
\begin{align*}
    R &= \{(T, F) \mid T \in \uptau_{k}, F \in \mathcal{B}_m \}, \text{ and } \\ 
    R_{b} &= \{(T, F) \mid T \in \uptau_{k}, F \in \mathcal{B}_m, I_{q(T,F)} \geq b \}. 
\end{align*}
Then for any $m \in \N$
\begin{align*}
    \frac{|R_{b}|}{|R|} &\leq 2^{-b}.
\end{align*}
\end{cor}

\begin{proof}
The proof follows from the definition of active information of expectations and Theorem~\ref{PROBABILITY-GOOD-SEARCH-PROBLEM}. Note,
\begin{align}
    b &\leq -\logg\left(\frac{p}{q(T,F)}\right)
\end{align}
implies 
\begin{align}
    q(T,F) &\geq p2^b.
\end{align}
Since $I_{q(T,F)} \geq b$ implies $q(T,F) \geq p2^b$, the set of problems for which $I_{q(T,F)} \geq b$ can be no bigger than the set for which $q(T,F) \geq p2^b$. By Theorem~\ref{PROBABILITY-GOOD-SEARCH-PROBLEM}, the proportion of problems for which $q(T,F)$ is at least $p2^b$ is no greater than $p/(p2^b)$. Thus,
\begin{align}
    \frac{|R_{b}|}{|R|} &\leq \frac{1}{2^b}.\label{PQ-ACTIVE-BOUND}
\end{align}
\end{proof}

\begin{cor}(Conservation of Expected $I_{\tilde{q}}$)
Define 
\[
    I_{\tilde{q}} = -\logg p/\tilde{q}, 
\]
where $p$ is the per-query probability of success for uniform random sampling and $\tilde{q}$ is the per-query probability of success for an alternative search algorithm. Define 
\[
    \uptau_{k} = \{T \mid T \subset \Omega, |T| = k \in \N\}
\] 
and let $\mathcal{B}_m$ denote any set of binary strings, such that the strings are of length $m$ or less. Let 
\begin{align*}
    R &= \{(T, F) \mid T \in \uptau_{k}, F \in \mathcal{B}_m \}, \text{ and } \\ 
    R_{b} &= \{(T, F) \mid T \in \uptau_{k}, F \in \mathcal{B}_m, \E[I_{\tilde{q}}] \geq b \}. 
\end{align*}
Then for any $m \in \N$
\begin{align*}
    \frac{|R_{b}|}{|R|} &\leq 2^{-b}.
\end{align*}
\end{cor}

\begin{proof}
By Jensen's inequality and the concavity of $\logg(\tilde{q}/p)$ in $\tilde{q}$, we have
\begin{align*}
b   &\leq \E\left[-\logg\left(\frac{p}{\tilde{q}}\right)\right] \\
    &\leq \logg\left(\frac{\E\left[\tilde{q}\right]}{p}\right) \\
    &= -\logg\left(\frac{p}{q(T,F)}\right)\\
    &= I_{q(T,F)}.
\end{align*}
The result follows by invoking Corollary~\ref{cor:CONSERVATION-ACTIVE-INFORMATION-EXPECTATIONS}.
\end{proof}

\begin{lem}\label{PROBABILITY-GOOD-SEARCH-STRATEGY}(Maximum Proportion of Satisfying Strategies)
    Given an integer $1 \leq k \leq n$, a set $\mathcal{S} = \{\mathbf{s} : \mathbf{s} \in \{0, 1\}^n, \|\mathbf{s}\| = \sqrt{k}\}$ of all $n$-length $k$-hot binary vectors, a set $\mathcal{P} = \{P : P \in \mathbb{R}^n, \sum_j P_j = 1\}$ of discrete $n$-dimensional simplex vectors, and a fixed scalar threshold $\epsilon \in [0, 1]$, then
\[
    \max_{\mathbf{s} \in \mathcal{S}}\frac{\mu(\mathcal{G}_{\mathbf{s}, \epsilon})}{\mu(\mathcal{P})} \leq \frac{1}{\epsilon}\frac{k}{n}
\]
where $\mathcal{G}_{\mathbf{s}, \epsilon} = \{P : P \in \mathcal{P}, \mathbf{s}^\top P \geq \epsilon \}$ and $\mu$ is Lebesgue measure. 
\end{lem}

\begin{proof}
    Similar results have been proved by others with regard to No Free Lunch theorems~\cite{WOLPERTMACREADY,COST-OF-SUCCESS,CUP,GENERAL-THEORY-OF-SEARCH,english1996evaluation}. Our result concerns the maximum proportion of sufficiently good strategies (not the mean performance of strategies over all problems, as in the NFL case) and is a simplification over previous search-for-a-search results.

    For $\epsilon = 0$, the bound holds trivially. For $\epsilon > 0$, We first notice that the $\mu(\mathcal{P})^{-1}$ term can be viewed as a uniform density over the region of the simplex $\mathcal{P}$, so that the integral becomes an expectation with respect to this distribution, where $P$ is drawn uniformly from $\mathcal{P}$. Thus, for any $\mathbf{s} \in \mathcal{S}$,
\begin{align*}
        \frac{\mu(\mathcal{G}_{\mathbf{s}, \epsilon})}{\mu(\mathcal{P})} 
        &= \int_{\mathcal{P}} \frac{1}{\mu(\mathcal{P})} \left[\1_{\mathbf{s}^\top P \geq \epsilon}\right] d\mu(P) \\
        &= \mathbb{E}_{{P \sim \mathcal{U}(\mathcal{P})}}\left[\1_{\mathbf{s}^\top P \geq \epsilon} \right] \\
        &= \Pr(\mathbf{s}^\top P \geq \epsilon) \\
        &\leq \frac{1}{\epsilon} \mathbb{E}_{{P \sim \mathcal{U}(\mathcal{P})}}\left[\mathbf{s}^\top P \right],
\end{align*}
where the final line follows from Markov's inequality. Since the symmetric Dirichlet distribution in $n$ dimensions with parameter $\alpha=1$ gives the uniform distribution over the simplex, we get 
\begin{align*}
    \mathbb{E}_{{P \sim \mathcal{U}(\mathcal{P})}}\left[ P \right] &= \mathbb{E}_{P \sim \text{Dir}(\alpha = 1)}\left[ P \right] \\
    &= \left(\frac{\alpha}{\sum_{i=1}^{n}\alpha}\right)\mathbf{1} \\
    &= \left(\frac{1}{n}\right)\mathbf{1}, 
\end{align*}
where $\mathbf{1}$ denotes the all ones vector. We have
\begin{align*}
    \frac{1}{\epsilon} \mathbb{E}_{{P \sim \mathcal{U}(\mathcal{P})}}\left[\mathbf{s}^\top P \right] 
        &= \frac{1}{\epsilon} \mathbf{s}^\top \mathbb{E}_{{P \sim \mathcal{U}(\mathcal{P})}}\left[ P \right] \\
        &= \frac{1}{\epsilon} \mathbf{s}^\top\left(\frac{1}{n}\right)\mathbf{1}\\
        &= \frac{1}{\epsilon}\frac{k}{n}.
\end{align*}
\end{proof}

\begin{thm}(Famine of Favorable Strategies)
    For any fixed search problem $(\Omega, t, f)$, set of probability mass functions $\mathcal{P} = \{P : P \in \mathbb{R}^{|\Omega|}, \sum_j P_j = 1\}$, and a fixed threshold $q_\text{min} \in [0, 1]$,
    \[
        \frac{\mu(\mathcal{G}_{t, q_\text{min}})}{\mu(\mathcal{P})} \leq \frac{p}{q_\text{min}},
    \]
    where $\mathcal{G}_{t, q_\text{min}} = \{P : P \in \mathcal{P}, t^\top P \geq q_\text{min} \}$ and $\mu$ is Lebesgue measure. Furthermore, the proportion of possible search strategies giving at least $b$ bits of active information of expectations is no greater than $2^{-b}$. 
\end{thm}
\begin{proof}
     Applying Lemma~\ref{PROBABILITY-GOOD-SEARCH-STRATEGY}, with $\mathbf{s} = t$, $\epsilon = q_\text{min}$, $k = |t|$, $n = |\Omega|$, and $p = \frac{|t|}{|\Omega|}$, yields the first result, while following the same steps as Corollary~\ref{cor:CONSERVATION-ACTIVE-INFO} gives the second (noting that by Lemma~\ref{EXPECTED-PER-QUERY-PERFORMANCE} each strategy is equivalent to a corresponding $q(t,f)$).
\end{proof}

\begin{thm}(Probability of Success Under Dependence)
Define $\uptau_{k} = \{T \mid T \subset \Omega, |T| = k \in \N\}$ and let $\mathcal{B}_m$ denote any set of binary strings, such that the strings are of length $m$ or less. Define $q$ as the expected per-query probability of success under the joint distribution on $T \in \uptau_{k}$ and $F \in \mathcal{B}_m$ for any fixed algorithm $\mathcal{A}$, so that $q := \E_{T, F}\left[q(T,F)\right]$, namely,
\[
    q = \E_{T, F}\left[\;\overline{P}(\omega \in T | F) \right] = \Pr(\omega \in T; \mathcal{A}).
\]
Then, 
\begin{align*}
        q \leq \frac{I(T; F) + D(P_T \| \mathcal{U}_T) + 1}{I_{\Omega}}
    \end{align*}
    where $I_{\Omega} = -\logg k/|\Omega|$, $D(P_T \| \mathcal{U}_T)$ is the Kullback-Leibler divergence between the marginal distribution on $T$ and the uniform distribution on $T$, and $I(T; F)$ is the mutual information. Alternatively, we can write
\begin{align*}
    q \leq \frac{H(\mathcal{U}_T) - H(T \mid F) + 1}{I_{\Omega}}
\end{align*}
where $H(\mathcal{U}_T) = \logg\binom{|\Omega|}{k}$.
\end{thm}

\begin{proof}
This proof loosely follows that of Fano's Inequality~\cite{fano1961transmission}, being a reversed generalization of it, so in keeping with the traditional notation we let $\X := \omega$ for the remainder of this proof. Let $Z = \1(\X \in T)$. Using the chain rule for entropy to expand $H(Z, T | \X)$ in two different ways, we get
\begin{align}
    H(Z, T | \X)  &= H(Z|T, \X) + H(T|\X) \\
                           &= H(T|Z, \X) + H(Z|\X).
\end{align}
By definition, $H(Z|T, \X) = 0$, and by the data processing inequality $H(T|F) \leq H(T|\X)$. Thus,
\begin{align}
    H(T | F)  &\leq H(T|Z, \X) + H(Z|\X).
\end{align}
Define $P_g = \Pr(\X \in T ; \mathcal{A}) = \Pr(Z=1)$. Then,
\begin{align}
H(T|Z, \X) &= (1-P_g)H(T|Z=0, \X) + P_gH(T|Z=1, \X) \\
                    &\leq (1-P_g)\logg\binom{|\Omega|}{k} + P_g\logg\binom{|\Omega|-1}{k-1} \\
                    &= \logg\binom{|\Omega|}{k} - P_g\logg\frac{|\Omega|}{k}.
\end{align}
We let $H(\mathcal{U}_T) = \logg\binom{|\Omega|}{k}$, being the entropy of the uniform distribution over $k$-sparse target sets in $\Omega$. Therefore,
\begin{align}
    H(T|F) &\leq H(\mathcal{U}_T) - P_g\logg\frac{|\Omega|}{k} + H(Z|\X).
\end{align}
Using the definitions of conditional entropy and $I_{\Omega}$, we get
\begin{align}
    H(T) - I(T; F) &\leq H(\mathcal{U}_T) - P_gI_{\Omega} + H(Z|\X),
\end{align}
which implies
\begin{align}
     P_gI_{\Omega} &\leq I(T; F) + H(\mathcal{U}_T) - H(T) + H(Z|\X) \\
        &= I(T; F) + D(P_T \| \mathcal{U}_T) + H(Z|\X).
\end{align}
Examining $H(Z|\X)$, we see it captures how much entropy of $Z$ is due to the randomness of $T$. To see this, imagine $\Omega$ is a roulette wheel and we place our bet on $\X$. Target elements are ``chosen'' as balls land on random slots, according to the distribution on $T$. When a ball lands on $\X$ as often as not (roughly half the time), this quantity is maximized. Thus, this entropy captures the contribution of dumb luck, being averaged over all $\X$. (When balls move towards always landing on $\X$, something other than luck is at work.) We upperbound this by its maximum value of 1 and obtain
\begin{align}
     \Pr(\X \in T ; \mathcal{A}) &\leq \frac{I(T; F) + D(P_T \| \mathcal{U}_T) + 1}{I_{\Omega}},
\end{align}
and substitute $q$ for $\Pr(\X \in T ; \mathcal{A})$ to obtain the first result, noting that $q = \E_{T, F}\left[\;\overline{P}(\X \in T | F) \right]$ specifies a proper probability distribution by the linearity and boundedness of the expectation. To obtain the second form, use the definitions $I(T; F) = H(T) - H(T|F)$ and $D(P_T \| \mathcal{U}_T) = H(\mathcal{U}_T) - H(T)$.
\end{proof}

\end{document}